\pdfoutput=1

\documentclass[11pt]{article}

\usepackage[]{EMNLP2023}

\usepackage{times}
\usepackage{latexsym}
\usepackage{booktabs}

\usepackage{amsmath,amsfonts,bm}

{}
{}
{}
{}

\def\eqref#1{equation~\ref{#1}}

\def\1{\bm{1}}

\def\vc{{\bm{c}}}

\def\vf{{\bm{f}}}

\def\vh{{\bm{h}}}

\def\vk{{\bm{k}}}

\def\vo{{\bm{o}}}

\def\vq{{\bm{q}}}

\def\vv{{\bm{v}}}
\def\vw{{\bm{w}}}
\def\vx{{\bm{x}}}
\def\vy{{\bm{y}}}
\def\vz{{\bm{z}}}

\def\mF{{\bm{F}}}

\def\mO{{\bm{O}}}

\def\mW{{\bm{W}}}

\def\mZ{{\bm{Z}}}

\DeclareMathAlphabet{\mathsfit}{\encodingdefault}{\sfdefault}{m}{sl}
\SetMathAlphabet{\mathsfit}{bold}{\encodingdefault}{\sfdefault}{bx}{n}

\usepackage{amsmath,amsfonts,amssymb}

\usepackage{amsthm}
\theoremstyle{definition}
\newtheorem{theorem}{Theorem}[section]
\newtheorem{proposition}[theorem]{Proposition}

\theoremstyle{remark}

\usepackage[T1]{fontenc}

\usepackage[utf8]{inputenc}

\usepackage{microtype}

\usepackage{inconsolata}

    \makeatletter
\def\@fnsymbol#1{\ensuremath{\ifcase#1\or \dagger\or \ddagger\or
   \mathsection\or \mathparagraph\or \|\or **\or \dagger\dagger
   \or \ddagger\ddagger \else\@ctrerr\fi}}
    \makeatother

\title{Practical Computational Power of Linear Transformers and\\ Their Recurrent and Self-Referential Extensions}

\author{Kazuki Irie$^{1}$\thanks{\hspace{1mm} Work done at IDSIA.} ~~ R\'obert Csord\'as$^{2}$ ~ Jürgen Schmidhuber$^{2,3}$\\
$^1$Harvard University ~
  $^2$The Swiss AI Lab IDSIA, USI \& SUPSI ~
  $^3$AI Initiative, KAUST \\
  {\texttt{kirie@fas.harvard.edu} ~ \texttt{\{robert,juergen\}@idsia.ch}}
}

\begin{document}
\maketitle

\begin{abstract}
Recent studies of the computational power of recurrent neural networks (RNNs) reveal a hierarchy of RNN architectures, given real-time and finite-precision assumptions. Here we study auto-regressive Transformers with linearised attention, a.k.a.~linear Transformers (LTs) or Fast Weight Programmers (FWPs). LTs are special in the sense that they are equivalent to RNN-like sequence processors with a fixed-size state, while they can also be expressed as the now-popular self-attention networks. We show that many well-known results for the standard Transformer directly transfer to LTs/FWPs. Our formal language recognition experiments demonstrate how recently proposed FWP extensions such as recurrent FWPs and self-referential weight matrices successfully overcome certain limitations of the LT, e.g., allowing for generalisation on the parity problem.
Our code is public.\footnote{ \url{https://github.com/IDSIA/fwp-formal-lang}}
\end{abstract}

\section{Introduction}
\label{sec:intro}
Recurrent neural networks (RNNs) are Turing-complete, given assumptions of infinite precision and unbounded computation (\citet{siegelmann91turing} for the sigmoid activation; \citet{ChenGMMK18} for ReLU; see also an alternative setting of \citet{ChungS21}).
This insight has a certain theoretical value, however, its implication for practical, real-world scenarios is limited.
In contrast, many recent works (e.g., \citet{weissGY18, suzgun2019lstm, MerrillWGSSY20, BhattamishraAG20}) strive to obtain results of practical relevance on the computational power of sequence-processing neural networks (NNs).
Typical assumptions of such studies are finite precision and ``input-bound'' computation (i.e., the number of computational steps is determined by and limited to the number of input tokens; a.k.a.~real-time assumption), essentially reflecting the reality of practical RNNs.
Although such analyses typically further assume models to be \textit{saturated} (\citet{MerrillWGSSY20}; which is not always the case for real-world RNNs),
they provide useful results and insights that can be empirically confirmed, e.g., on formal language recognition tasks.
For example, the long short-term memory (LSTM; \citet{hochreiter1997long}) can learn to solve and generalise on certain counter-language tasks (\citet{gers2001lstm,SchmidhuberGE02,weissGY18, suzgun2019lstm}), unlike its simplified versions such as the gated recurrent unit (GRU; \citet{cholearning, gers2000learning}), the Quasi-RNN (\citet{Bradbury17}; see \citet{MerrillWGSSY20}) or the simple sigmoid RNN \citep{elman1988finding}.
Tools from the theory of computation (see, e.g., \citet{sipser1996introduction}) provide a formalism to compare the practical computational power of these NN architectures, allowing for categorisation into a hierarchy (\citet{MerrillWGSSY20}; see Appendix \ref{app:automata} for further references).

More recent works \citep{BhattamishraAG20,BhattamishraPG20,EbrahimiGZ20,YaoPPN20} derive results for the now-popular Transformer \citep{trafo}.
For example, \citet{BhattamishraAG20} show (we review in Sec.~\ref{sec:power}) that Transformers can learn to solve and generalise on certain counter-language tasks (including the context-sensitive $a^nb^nc^n$),
while they fail to learn certain simple regular languages (many examples of which are \textit{non-star-free} languages) including the parity task (see also \citet{Chiang22}).
For further theoretical studies, we refer to, e.g., \citet{hahn2020theoretical, merrill2022saturated, HaoAF22}.\looseness=-1

Here we focus on auto-regressive Transformers with \textit{linearised attention} a.k.a.~linear Transformers (LTs; \citet{katharopoulos2020transformers,choromanski2020rethinking,peng2021random}) or Fast Weight Programmers (FWPs; \citet{Schmidhuber:91fastweights,schmidhuber1992learning, schlag2021linear,irie2021going}).
LTs are special, as they can be equivalently expressed as an RNN-like sequence processor with a constant-size state (the FWPs from the 1990s; \citet{Schmidhuber:91fastweights, schmidhuber1992learning,schmidhuber1993reducing}), while they are originally formulated as self-attention networks.
This property is interesting  for discussions of computational power,
since it removes one of the classic distinctions between Transformers and RNNs: RNNs are ``automata-like'' constant-size stateful models, while Transformers are not.
Building upon \citet{weissGY18,BhattamishraAG20, MerrillWGSSY20}, we show that many existing results on Transformers directly transfer to LTs/FWPs.
Also, prior work proposes several extensions to LTs inspired by its RNN-like form, including recurrence \citep{irie2021going} and self-reference \citep{IrieSCS22}, showing their practical benefits on actual real-world tasks (e.g., reinforcement learning in game environments).
Here we demonstrate how 
both recurrent and self-referential extensions enhance LTs' practical computational power on formal language recognition tasks.

\section{Background}
\label{sec:background}
Here we briefly review LTs and FWPs.
Let $d_\text{in}$, $d_\text{out}$, $d_\text{key}$, $t$ be positive integers, and $\otimes$ denote outer product.
An FWP \citep{Schmidhuber:91fastweights,schmidhuber1992learning} is a sequence-processing NN that, at each time step $t$, transforms
an input $\vx_t \in \mathbb{R}^{d_\text{in}}$ to an output $\vy_t \in \mathbb{R}^{d_\text{out}}$ as follows:
\begin{align}
\label{eq:proj}
\vq_t, \vk_t, \vv_t &= \mW^\text{slow}\vx_t \\
\label{eq:update}
\mW_t &= \mW_{t-1} +\vv_t \otimes \phi(\vk_t) \\
\vy_t  &= \mW_t \phi(\vq_t) \label{eq:get}
\end{align}
where $\vk_t, \vq_t \in \mathbb{R}^{d_\text{key}}$, $\vv_t \in \mathbb{R}^{d_\text{out}}$, and $\mW^\text{slow} \in \mathbb{R}^{(2 d_\text{key} + d_\text{out}) \times d_\text{in}}$ is a weight matrix (the ``slow'' weights), and $\phi$ is an activation function.
The ``fast'' weight matrix $\mW_t \in \mathbb{R}^{d_\text{out} \times d_\text{key}}$ is initially set to 0, i.e., $\mW_0=0$.
This can be viewed as a system of two NNs where one net (the slow net; Eq.~\ref{eq:proj}) learns to ``program'' the fast net (Eq.~\ref{eq:get}) by generating its weight changes (Eq.~\ref{eq:update}).
This $\vx_t$-to-$\vy_t$ transformation can be also expressed as \textit{linear attention} using $\phi(\vk_t), \vv_t, \phi(\vq_t)$ as key, value, query vectors (\citet{katharopoulos2020transformers, schlag2021linear, ba2016using}; see also our brief review in Appendix \ref{app:review}).
To be more specific, such FWPs correspond to \textit{unnormalised} LTs (ULTs).
LTs with \textit{normalised} linear attention (NLTs; \citet{katharopoulos2020transformers,choromanski2020rethinking,peng2021random}) use the following additional computation:
\begin{align}
\label{eq:z_update}
\vz_t &= \vz_{t-1} + \phi(\vk_t) \\
\vy_t  &= \frac{1}{\vz_t \cdot \phi(\vq_t)} \mW_t \phi(\vq_t) \label{eq:fw_get}
\end{align}
where $\vz_t \in \mathbb{R}^{d_\text{key}}$ with $\vz_0=0$, and $\cdot$ denotes dot product, i.e., $\vz_t \cdot \phi(\vq_t) \in \mathbb{R}$, replacing Eq.~\ref{eq:get}.
In practice, this normalisation can be removed without loss of performance \citep{schlag2021linear, irie2021going}, which is convenient as no extra vector $\vz_t$ needs to be stored.
All standard Transformer components including feedforward blocks, residual connections, and layer-norm are used in LTs.\looseness=-1

This equivalence has inspired a series of extensions to Transformers.
Here we highlight three such examples:
delta-rule, recurrence, and self-reference, which we study in Sec.~\ref{sec:exp}.\looseness=-1

\paragraph{Delta-rule.}
\citet{schlag2021linear} replace the purely additive update rule of Eq.~\ref{eq:update} by the classic delta-rule for error correction \citep{widrow1960adaptive}; Eq.~\ref{eq:update_delta} below.
The slow weight matrix in the resulting model, called DeltaNet, is 
$\mW_\text{slow} \in \mathbb{R}^{( 2 * d_\text{key} + d_\text{out} + 1) \times d_\text{in}}$ that also generates a dynamic learning rate $\beta_t \in \mathbb{R}$ (to which we apply the sigmoid function $\sigma$).
With the delta rule, $\phi$'s output elements need to be positive and sum up to one (e.g., we use softmax) for stability.
\begin{align}
\label{eq:proj_beta}
\vq_t, \vk_t, \vv_t, \beta_t &= \mW^{\text{slow}}\vx_t\\
\label{eq:retrieve_delta}
\bar{\vv}_t &= \mW_{t-1} \phi(\vk_t) \\
\label{eq:update_delta}
\mW_t = \mW_{t-1} + \sigma(\beta_t)&(\vv_t - \bar{\vv}_t) \otimes \phi(\vk_t)
\end{align}
Note that this introduces an extra dependency to LTs; the update term $\sigma(\beta_t) (\vv_t - \bar{\vv}_t) \otimes \phi(\vk_t)$ in Eq.~\ref{eq:update_delta} is a function of $\mW_{t-1}$ unlike in Eq.~\ref{eq:update}.
We'll empirically illustrate how this modification introduces an explicit forget mechanism to the LT, using the ``reset Dyck-1'' language (Sec.~\ref{sec:exp}).

\paragraph{Recurrence.}
One trivial extension to the LTs/DeltaNet above is to add ``proper recurrent connections'' by feeding back the output $\vy_{t-1}$ from the previous step $t-1$ as an input at step $t$
(for other recurrent extensions we refer to \citet{irie2021going}).
The resulting Recurrent DeltaNet \citep{irie2021going} (or Recurrent Delta) is obtained by replacing Eq.~\ref{eq:proj_beta} in the DeltaNet by:
\begin{align}
\vq_t, \vk_t, \vv_t, \beta_t &= \mW^\text{slow} [\vx_t,  \tanh(\vy_{t-1})]
\label{eq:rec}
\end{align}
where $\mW_\text{slow} \in \mathbb{R}^{(2 * d_\text{key} + d_\text{out} + 1) \times (d_\text{in} + d_\text{out})}$, and $[\vx_t,  \tanh(\vy_{t-1})] \in \mathbb{R}^{d_\text{in} + d_\text{out}}$ denotes concatenation of the two vectors.

\paragraph{Self-Reference.}
Another extension of the DeltaNet above is the \textit{modern} self-referential weight matrix (SRWM; \citet{IrieSCS22}).
Motivated by recursive self-improvement \citep{good1965,schmidhuber1987evolutionary} and the original SRWM \citep{Schmidhuber:92selfref}, \citet{IrieSCS22} extend the FWP that generates weight changes for another NN to obtain an NN that modifies itself.
At each time step $t$, an SRWM $\mW_{t-1} \in \mathbb{R}^{(d_\text{out} + 2 * d_\text{in} + 1) \times d_\text{in}}$ transforms an input $\vx_t \in \mathbb{R}^{d_\text{in}}$ to an output $\vy_t \in \mathbb{R}^{d_\text{out}}$, and updates itself to $\mW_{t}$ as follows:
\begin{align}
\label{eq:srm_start}
\vy_t, \vk_t, \vq_t, \beta_t &= \mW_{t-1} \vx_t \\
\label{eq:srm_key}
\vv_t = \mW_{t-1} \phi(\vq_t)
&; \, \bar{\vv}_t = \mW_{t-1} \phi(\vk_t) \\
\label{eq:srm_end}
\mW_{t} = \mW_{t-1} + \sigma(\beta_t)&(\vv_t - \bar{\vv}_t) \otimes \phi(\vk_t)
\end{align}
where $\vv_t, \bar{\vv}_t \in \mathbb{R}^{(d_\text{out} + 2 * d_\text{in} + 1)}$, $\vq_t, \vk_t \in \mathbb{R}^{d_\text{in}}$
 and $\beta_t \in \mathbb{R}$.
 The initial values $\mW_0$ are the only trainable parameters of this layer.
Note that it is straightforward to further extend this model with recurrence as in Eq.~\ref{eq:rec}.
Here we focus on this specific self-referential extension.

\section{Expressive Power of LTs}
\label{sec:power}
Here we revisit several existing results on the practical computational power of Transformers for normalised LTs (NLTs; Eqs.~\ref{eq:proj}-\ref{eq:update};\ref{eq:z_update}-\ref{eq:fw_get}).
Results in this section directly build upon prior work \citep{BhattamishraAG20, MerrillWGSSY20}.
As we'll see, some of the results are not obvious from Eqs.~\ref{eq:proj}-\ref{eq:update};\ref{eq:z_update}-\ref{eq:fw_get}.
However, their connection to Transformers allows us to trivially derive them.
While one can come up with certain custom positional encoding methods that empower Transformers to specifically recognise certain languages, here we focus on generic Transformers without positional encoding \citep{BhattamishraAG20,irie19:trafolm,tsai2019}.

We start by noticing that the hidden ``state'' update of NLTs (Eq.~\ref{eq:update}) is element-wise.
This is reminiscent of simplified LSTMs such as Quasi-RNNs \citep{Bradbury17}, which are known to be limited \citep{MerrillWGSSY20}: in particular, Quasi-RNNs are rationally recurrent (\citet{PengSTS18}).
However, we have the following result:

\begin{proposition}[``Rational Recurrence'']
\label{prop:rr}
\textit{NLTs are not rationally recurrent.}
\end{proposition}
\begin{proof}
\citet{MerrillWGSSY20}'s proof by construction for their Theorem 15 remains valid for NLTs.
\end{proof}

It should be noted that the actual output of NLTs is $\vy_t$ (Eq.~\ref{eq:fw_get}), not $\mW_t$.
In fact, NLTs can recognise certain counter languages, inheriting the properties of the Transformer:

\begin{proposition}[``Simple Counter Languages'']
\label{prop:count}
\textit{NLTs can recognise certain counter languages.}
\end{proposition}
\begin{proof}
\citet{BhattamishraAG20}'s proof for their Proposition 4.1 is valid for NLTs:  NLTs can recognise Shuffle-Dyck languages.
\end{proof}

However, similar to Transformers, NLTs are fundamentally limited:

\begin{proposition}[``Regular Languages'']
\label{prop:reg}
\textit{NLTs can not recognise certain regular languages.}
\end{proposition}
\begin{proof}
\citet{BhattamishraAG20}'s proof for their Lemma C.4 remains valid for NLTs; NLTs can not recognise the regular language $(aa)^*$.
\end{proof}

Finally, we comment on the ``state complexity'' as defined by \citet{MerrillWGSSY20}:

\begin{proposition}[``State complexity'']
\label{prop:state}
\textit{The state complexity of a single-layer NLT is $O(\log(n))$ (same as the regular self-attention and LSTM).}
\end{proposition}
\begin{proof}
\citet{MerrillWGSSY20}'s proof for Theorem 16 remains valid for normalised linear attention.
\end{proof}

Given the original proofs by \citet{BhattamishraAG20} and \citet{MerrillWGSSY20}, the proofs above are straightforward for \textit{normalised} LTs.
For further discussions on these proofs and their extension for \textit{unnormalised} LTs (Eqs.~\ref{eq:proj}-\ref{eq:get}), we refer to Appendix \ref{app:proofs}.
In sum, these statements on the expressiveness of Transformers remain valid for both normalised and unnormalised LTs.

\begin{table*}[t]
    \centering
    \small
    \caption{Accuracies of various models on the formal language recognition tasks.}
    \label{tab:main}
\begin{tabular}{lrrrrrrrrrrrrr}
\toprule
 & \multicolumn{6}{c}{Non-Star-Free Regular} & \multicolumn{6}{c}{Counter} \\  \cmidrule(r{.45em} l{.50em}){2-7}  \cmidrule(r{.45em} l{.50em}){8-13}
 & \multicolumn{2}{c}{Parity} & \multicolumn{2}{c}{$(aa)^*$}  & \multicolumn{2}{c}{$(abab)^*$} & \multicolumn{2}{c}{$a^nb^n$} & \multicolumn{2}{c}{$a^nb^nc^n$} & \multicolumn{2}{c}{Shuffle-2}\\  \cmidrule(r{.45em} l{.50em}){2-7}  \cmidrule(r{.45em} l{.50em}){8-13}
Model &  Bin0 & Bin1 & Bin0 & Bin1 & Bin0 & Bin1 & Bin0 & Bin1 & Bin0 & Bin1 & Bin0 & Bin1 
 \\
\midrule
LSTM & \textbf{100.0} & \textbf{100.0} & \textbf{100.0} & \textbf{100.0} & \textbf{100.0} & \textbf{100.0} & \textbf{100.0} & \textbf{100.0} & \textbf{100.0} & \textbf{100.0} & \textbf{100.0} & \textbf{100.0} \\
e-LSTM & \textbf{100.0} & \textbf{100.0} & \textbf{100.0} & \textbf{100.0} & \textbf{100.0} & \textbf{100.0} & \textbf{100.0} & 90.0 & \textbf{100.0} & 22.0 & \textbf{100.0} & 85.7 \\
Transformer & 47.1 & 0.1 & 0.0 & 0.0 & 0.0 & 0.0 & \textbf{100.0} & \textbf{100.0} & \textbf{100.0} & \textbf{100.0} & \textbf{100.0} & \textbf{100.0} \\ \midrule
Linear & 77.9 & 0.2 & 0.0 & 0.0 & 0.0 & 0.0 & \textbf{100.0} & \textbf{100.0} & \textbf{100.0} & \textbf{100.0} & \textbf{100.0} & \textbf{100.0}  \\
DeltaNet & 97.3 & 11.8 & 0.0 & 0.0 & 0.0 & 0.0 & \textbf{100.0} & \textbf{100.0} & \textbf{100.0} & \textbf{100.0} & \textbf{100.0} & \textbf{100.0} \\
Recurrent Delta & \textbf{100.0} & \textbf{100.0} & \textbf{100.0} & \textbf{100.0} & \textbf{100.0} & \textbf{100.0} & \textbf{100.0} & \textbf{100.0} & \textbf{100.0} & \textbf{100.0} & \textbf{100.0} & \textbf{100.0}  \\
SRWM & \textbf{100.0} & \textbf{100.0} & \textbf{100.0} & \textbf{100.0} & \textbf{100.0} & \textbf{100.0} & \textbf{100.0} & \textbf{100.0} & \textbf{100.0} & \textbf{100.0} & \textbf{100.0} & \textbf{100.0} \\
\bottomrule
\end{tabular}
\end{table*}

\section{Experiments}
\label{sec:exp}

Here we provide several empirical results on capabilities and limits of \textit{unnormalised} LTs/FWPs and their extensions, using formal languages.\looseness=-1

\subsection{Tasks}
We evaluate LT models on formal language recognition tasks using several non-star-free regular languages---parity, $(aa)^*$, $(abab)^*$---, and counter languages---$a^nb^n$, $a^nb^nc^n$, Shuffle-2, and reset Dyck-1.
This choice is guided by \citet{BhattamishraAG20}'s results on the standard Transformers to specifically evaluate LTs' capabilities and limits.
Following prior work \citep{gers2001lstm},
for $a^nb^n$ and $a^nb^nc^n$,
we define the language recognition task 
as the next character prediction task.
For example, for the context-free language $a^nb^n$, if the input to the model is \texttt{aaabbb}, the model has to output \texttt{NNNbbS} where \texttt{N} denotes ``\textit{cannot-predict-yet}'' token, and \texttt{S} denotes the sequence-end token.
Appendix \ref{app:exp} contains corresponding descriptions for other tasks.
We train models on positive examples of up to a certain length, and validate them on positive examples with longer lengths.
We denote the corresponding data splits as ``Bin0'' (sequences with lengths seen during training) and ``Bin1'' (longer sequences).\looseness=-1

For evaluation, for each position of the sequence,
we define the model prediction as the token that is the most likely according to the model.
We report accuracy on the sequence level; we count a sequence as correctly recognised only if the model prediction is correct for all tokens in the sequence.
Further experimental details (e.g., hyper-parameters) can be found in Appendix \ref{app:exp}.

\subsection{Results}

\paragraph{Main Results.}
Table \ref{tab:main} shows our main results.
The top part of the table shows three reference baselines: LSTM, Transformer, and an LSTM with element-wise recurrence and tied input-forget gate (denoted as e-LSTM; see details in Appendix \ref{app:elstm}).
Table \ref{tab:main}/Left shows the results for the regular language tasks: parity, $(aa)^*$, and $(abab)^*$,
on which \citet{BhattamishraAG20} report Transformers to fail.
Inheriting their properties, the standard LT fails on all these tasks (recall, however, as stated by \citet{BhattamishraAG20}: non-star-free regular languages are not the strict set of regular languages on which Transformers fail).
We also confirm that the delta rule is not enough to help LTs succeed in these tasks.
In contrast, both the recurrent (Recurrent Delta) and self-referential (SRWM) extensions successfully solve and generalise on these tasks.\looseness=-1

We also confirm that all LT variants can learn representative counter languages that the original Transformer can learn (Table \ref{tab:main}/Right).
One interesting empirical trend (not reflected in Table \ref{tab:main}) is that the base Transformer and LT tend to more easily find solutions that generalise.
For DeltaNet, Recurrent Delta and SRWM, many configurations achieve 100\% accuracy on Bin0, without achieving exactly 100\% on Bin1.

\paragraph{Reset Dyck-1.}
The main experiment above does not emphasise  the benefits of the delta-rule which by itself cannot help LTs to recognise parity, $(aa)^*$ and $(abab)^*$.
However, the delta rule is typically reported to be crucial across many practical tasks (including reinforcement learning in game environments \citep{irie2021going}, image generation \citep{irie2023image}, or long time-series classification \citep{irie2022neural}).
Here we use ``reset Dyck-1'' to illustrate its benefits.
\citet{BhattamishraAG20} prove that a one-layer self-attention network cannot recognise reset Dyck-1 as it has no mechanism to rewrite memory.
As shown in Table \ref{tab:reset}, the delta-rule by itself allows LTs to recognise this language.

\begin{table}[t]
    \centering
    \small
    \caption{Accuracies of single-layer models.}
    \label{tab:reset}
\begin{tabular}{lrrrr}
\toprule
 & \multicolumn{2}{c}{Dyck-1} & \multicolumn{2}{c}{Reset Dyck-1} \\  \cmidrule(r{.45em} l{.50em}){2-3} \cmidrule(r{.45em} l{.50em}){4-5}
Model &  Bin0 & Bin1 &  Bin0 & Bin1 \\
\midrule
Linear & \textbf{100.0} & \textbf{100.0} & 44.5 & 41.1 \\
DeltaNet & \textbf{100.0} & \textbf{100.0} & \textbf{100.0} & \textbf{100.0}  \\
Recurrent Delta & \textbf{100.0} & \textbf{100.0} & \textbf{100.0} & \textbf{100.0}  \\
SRWM &\textbf{100.0} & \textbf{100.0} & \textbf{100.0} & \textbf{100.0}  \\
\bottomrule
\end{tabular}
\end{table}

\section{Outlook: Self-Modifying Automata}
\label{app:self_mod_fa}
Here we discuss a potentially interesting perspective for further studying SRWM models.
Self-Modifying Automata (SMAs) and Self-Modifying Finite Automata (SMFAs;
\citet{rubinstein1993self, rubinstein1995self, shutt1995self, wang1999note})\footnote{These SMFAs are different from another model with the same name by
\citet{moulin1992modifiable, moulin1999very, moulin2006adaptive} which is not a formal model of computation, unlike SMFAs.}
are ``FAs'' with capabilities to modify their transition function at runtime.
Despite its name containing ``finite automata,'' they are provably computationally more powerful than FAs: the least restricted versions thereof are Turing-complete, certain restricted ones can still recognise certain context-sensitive languages (meta-linear languages).
It may be interesting to connect SMFAs with SRWMs in future work.
For example, we may try to extract SMFAs from SRWMs trained on certain meta-linear/counter languages.

\section{Conclusion}
\label{sec:ccl}
We discuss the computational power of Transformers with linearised attention in the context of formal languages.
We show that such linear Transformers (LTs) a.k.a.~Fast Weight Programmers (FWPs) inherit several capabilities and limitations of Transformers.
We demonstrate that their recurrent and self-referential extensions successfully overcome some of those limitations.
We hope this work will inspire the development of formally more powerful Transformer models.\looseness=-1

\section*{Limitations}
Our study follows prior work on the
computational power of RNNs, in particular, \citet{BhattamishraAG20}'s results on Transformers, and \citet{MerrillWGSSY20}'s discussion of state-complexity and rational recurrence.
Naturally, this is not an exhaustive study of LT properties, and we cannot definitively compare the expressivity of LTs to the one of standard Transformers solely based on what is presented here.
Also, while it is rather obvious that the recurrent extension enhances the computational power of LTs, future work should provide more insights on the power of self-reference; see also our outlook Sec.~\ref{app:self_mod_fa} on \textit{self-modifying automata}.
A comparison of such models to memory-augmented RNNs \citep{graves2016hybrid,suzgun2019memory,deletang2022neural} is also left for future work.

We focus on eight tasks by \citet{BhattamishraAG20} to illustrate the capabilities and limitations of Transformers (and thus, those of LTs).
Future work will extend experimental results to more diverse tasks, e.g., those presented in \citet{deletang2022neural}.\looseness=-1

\section*{Acknowledgements}
This research was partially funded by ERC Advanced grant no: 742870, project AlgoRNN,
and by Swiss National Science Foundation grant no: 200021\_192356, project NEUSYM.
We are thankful for hardware donations from NVIDIA and IBM.
The resources used for this work were partially provided by Swiss National Supercomputing Centre (CSCS) project d123.

\bibliography{references}
\bibliographystyle{acl_natbib}

\appendix

\section{RNNs and Theory of Computation}
\label{app:automata}
RNNs have been related to Finite Automata (FAs) for many decades
\citep{mcculloch1943logical, kleene1956representation,cleeremans1989finite, siegelmann1996recurrent,WeissGY19,korsky2019computational}.
Many works explore the extraction of FAs from trained RNNs 
\citep{GilesMCCSL92, DasM93, Kolen93, OmlinG96, GilesOT99,WeissGY18fa} (note that our work hints at the possibility to extract FAs also from LTs).
Others use synthetic and formal languages to benchmark RNNs \citep{allen1990connectionist, schmidhuber1999evaluating}.
The connection between RNNs and FAs also motivates certain architectural enhancements of RNNs, such as stack-augmented RNNs \citep{pollack1990recursive, das1992learning, sun1993neural, JoulinM15, GrefenstetteHSB15,dusellchiang2020learning}.

For further references on theoretical works studying RNNs, see also \citet{merrill2019sequential, merrill2020linguistic, ackerman2020survey},
and for Transformers, see also \citet{hahn2020theoretical,WeissGY21,liu2022transformers,YunBRRK20,perez2021attention}.

\section{Review of Linear Attention vs.~FWPs}
\label{app:review}
Here we briefly review the formal connection between the ``RNN-form'' of LTs shown in Sec.~\ref{sec:background} and attention \citep{katharopoulos2020transformers, schlag2021linear, ba2016using}.
Starting from Eq.~\ref{eq:fw_get}, using the definition of $\mW_t$ (Eq.~\ref{eq:update}) and $\vz_t$ (Eq.~\ref{eq:z_update}),
\begin{align}
\vy_t  &= \frac{1}{\vz_t \cdot \phi(\vq_t)} \mW_t \phi(\vq_t) \\
& = \dfrac{\left(\sum_{\tau=1}^t \vv_{\tau} \otimes \phi(\vk_{\tau})\right) \phi(\vq_t)}{\left(\sum_{\tau'=1}^t \phi(\vk_{\tau'})\right)^{\intercal} \phi(\vq_t)}\\
& = \dfrac{\sum_{\tau=1}^t \vv_{\tau} \phi(\vk_{\tau})^{\intercal} \phi(\vq_t)}{\sum_{\tau'=1}^t \phi(\vk_{\tau'})^{\intercal} \phi(\vq_t)} \\
& = \sum_{\tau=1}^t \alpha_{t,\tau} \vv_{\tau}
\end{align}
where $\alpha_{t,\tau} = \dfrac{\phi(\vk_{\tau})^{\intercal} \phi(\vq_t)}{\sum_{\tau'=1}^t \phi(\vk_{\tau'})^{\intercal} \phi(\vq_t)}$.
We can recognise that this is effectively attention with \textit{normalised} weights $\alpha_{t,\tau}$ using  $\phi(\vk_t), \vv_t, \phi(\vq_t)$ as key, value, query vectors.
For LTs/FWPs without normalisation of Eq.~\ref{eq:fw_get} (i.e., Eqs.~\ref{eq:proj}-\ref{eq:get}), the derivation above is similar, but the corresponding attention weights are not normalised.

Note that this relation is analogous to the one that connects the perceptron to kernel machines \citep{aizerman1964theoretical,irie2022dual}.

\section{Further details}

\subsection{Further Comments on Proofs}
\label{app:proofs}

Here we provide some more comments on the proofs of our propositions presented in Sec.~\ref{sec:power}.

For both Proposition \ref{prop:rr} and Proposition \ref{prop:count}, the original proofs, i.e., the proof of \citet{MerrillWGSSY20}'s Theorem 15 and that of \citet{BhattamishraAG20}'s Proposition 4.1 respectively, consist in constructing a self-attention layer capable of solving certain counting tasks---a task checking whether two alphabets appear the same number of times in a sequence, in the former, and the Shuffle-$k$ languages in the latter.
In both cases, as softmax is not explicitly required for computing normalised attention weights,
the proofs directly remain valid for NLTs.
Now, the question is whether they still hold for ULTs:
is the normalisation of attention weights required?
The answer is \textit{no} in both constructions.
The core function of the layer consists in counting and comparing the occurrence of two symbols (e.g., opening and closing brackets in the case of Shuffle-$k$).
The actual comparison is done by computing the difference between the two counts and comparing it against zero.
This function is preserved without normalisation of attention weights.
Therefore, these proofs can be directly adopted for ULTs.\looseness=-1

For Proposition \ref{prop:reg}, the original proof of \citet{BhattamishraAG20}'s Lemma C.4 shows that Transformers without positional encoding can not recognise $(aa)^*$ because the output of the Transformer is the same for all steps for this language defined using a single symbol. There is no way to distinguish between odd and even steps, which is essential to recognise $(aa)^*$.
While this remains true for normalised linear attention resulting in uniform attention weights, this argument does not directly hold for unnormalised variants.
Nevertheless, if we assume an extra layer normalisation layer following the self-attention layer (which is typically the case in practice), the constant-output argument also holds for unnormalised linear attention.\looseness=-1

Finally, for Proposition \ref{prop:state}, \citet{MerrillWGSSY20}'s proof for their Theorem 16 consists in ``counting'' the number of configurations of layer activations. This is independent of normalisation schemes, and remains valid for both unnormalised and normalised linear Transformers.

\subsection{Experimental Details}
\label{app:exp}

\begin{table*}[t]
\caption{
Hyper-parameter search space.
}
\label{tab:hyperparameters}
\begin{center}
\begin{tabular}{rc}
\toprule
Parameters & Values  \\ \midrule
Number of layers & \{1, 2, 4\}  \\
Hidden size & \{8, 16, 32\} \\
Feedforward block multiplier & \{1, 2, 4\} \\
Number of heads & \{1, 2, 4\} \\
Learning rate & \{1e-2, 2e-2, 3e-2, 1e-3, 2e-3, 3e-3\} \\
Batch size & \{16, 32, 64\} \\
\bottomrule
\end{tabular}
\end{center}
\end{table*}

\paragraph{Task Definition.}
For parity, $(aa)^*$, and $(abab)^*$,
the output to be predicted at each step is the result for the prefix presented so far.
For example for parity, if the input sequence is \texttt{0010}, the output sequence should be \texttt{TTFF}, where \texttt{T} and \texttt{F} denote `true' and `false' for parity.
For Shuffle-2 (\citet{suzgun2019lstm}: a mixture of two Dyck-1 languages, i.e., with two kinds of parentheses), we encode the task as follows.
By denoting the parenthesis \texttt{[]} as type-0 and \texttt{()} as type-1\footnote{Here there was a typo in the camera ready version (the definitions of type-0 and type-1 were swapped).}, 
we consider four cases: `0' (both are closed), `1' (type-0 is open), `2' (type-1 is open), `3' (both are open);
for example, for the input \texttt{([])}, the output should be \texttt{2320}.\looseness=-1

\paragraph{Dataset.}
We use the official (pre-generated) dataset made publicly available by \citet{BhattamishraAG20}, except for reset Dyck-1 which we generate ourselves using their official public code.
``Bin0'' split contains sequences of ``lengths'' shorter than 50, while ``Bin1'' contains those with ``lengths'' between 51 and 100.
The exact definition of ``length'' above depends on the tasks; for tasks such as parity, it directly refers to the actual sequence length; for tasks such as $a^nb^nc^n$, it refers to $n$, i.e., the actual lengths of Bin0 sequences are up to 150 for $a^nb^nc^n$, while they are between 153 and 300 for Bin1 sequences.

\textbf{Hyper-parameter search} spaces for all Transformer family models (i.e., all models except LSTM and e-LSTM) are shown in Table \ref{tab:hyperparameters}.
Note that ``Feedforward block multiplier'' refers to the factor $N_\text{ff}$ that relates the hidden size $d_\text{model}$ of the Transformer to its feedforward up-projection size $d_\text{ff}$, i.e., $d_\text{ff} = N_\text{ff} d_\text{model}$.
For LSTM and e-LSTM, we use the same search space except
that the number of layers is in \{1, 2\}, and the hidden size is in \{8, 16, 32, 64\}, and irrelevant parameters (i.e., the feedforward block multiplier and the number of heads) are ignored.
The reported results are the best performance across all the hyper-parameter search, as done in previous work \citep{BhattamishraAG20}. 
Tables \ref{tab:hyperparameters_best_rec_delta} and \ref{tab:hyperparameters_best_srwm} display the best hyper-parameter configurations on each task for Recurrent Delta and SRWM models, respectively.
For further details, we refer to our public code.

Any other configurations for the SRWM follow those of \citet{IrieSCS22}, except that we initialise the `query' projection sub-matrix in the self-referential weight matrix using a normal distribution with a mean value of 0 and a standard deviation of $0.01 / \sqrt{d_\text{head}}$ while other sub-matrices use an std of $1 / \sqrt{d_\text{head}}$ (this is motivated by the fact that a generated query vector is immediately multiplied with the same SRWM to produce a value vector).

\begin{table*}[t]
\small
\caption{
Best hyper-parameters for Recurrent Delta. When there are more than one best configurations, we report the one that converges the fastest.
}
\label{tab:hyperparameters_best_rec_delta}
\begin{center}
\begin{tabular}{rcccccccc}
\toprule
Parameters & Parity & $(aa)^*$  & $(abab)^*$ & $a^nb^n$ & $a^nb^nc^n$ & Shuffle-2 & Dyck-1 & Reset Dyck-1 \\ \midrule
Number of layers & 1 & 1 & 2  & 1 & 1 & 4 & 1 & 1 \\
Hidden size & 4 & 8 & 8  &  8 & 16 & 16 & 16 & 8 \\
Feedforward block multiplier & 1 & 1 & 1 & 1 & 1 & 2 & 1 & 1 \\
Number of heads & 1 & 2 & 2 &  4 & 4 & 4 & 2 & 4  \\
Learning rate & 2e-2 & 2e-2 & 2e-2 & 3e-2 & 2e-2 & 2e-2 & 3e-2   & 3e-2 \\
Batch size & 16 & 16 & 16 & 16 & 16 & 32 & 16 & 16 \\
\bottomrule
\end{tabular}
\end{center}
\end{table*}

\begin{table*}[t]
\small
\caption{
Best hyper-parameters for SRWM. When there are more than one best configurations, we report the one that converges the fastest.
}
\label{tab:hyperparameters_best_srwm}
\begin{center}
\begin{tabular}{rcccccccc}
\toprule
Parameters & Parity & $(aa)^*$  & $(abab)^*$ & $a^nb^n$ & $a^nb^nc^n$ & Shuffle-2 & Dyck-1 & Reset Dyck-1 \\ \midrule
Number of layers & 1 & 2 & 1 & 2 & 1 & 1 & 1 & 1 \\
Hidden size & 8 & 16 & 16 & 16  & 8 & 8 & 16 & 16 \\
Feedforward block multiplier & 1 & 1 & 1 & 2 & 2 & 2 & 2 & 2 \\
Number of heads & 2 & 4 & 2 & 4  & 2 & 2 & 8 & 2 \\
Learning rate & 3e-2 & 2e-2  & 3e-2  & 1e-2 & 2e-2 & 3e-2 & 3e-2 & 3e-2  \\
Batch size & 16 & 16 & 16 & 16 & 16 & 32 & 16 & 16 \\
\bottomrule
\end{tabular}
\end{center}
\end{table*}

\subsection{Details of e-LSTM}
\label{app:elstm}
In the main text, we evaluate the element-wise LSTM with tied input-forget gates (e-LSTM; \citet{irie2023exploring}) as an illustrative example of computationally limited RNNs.
e-LSTM is essentially an LSTM with only element-wise recurrence, which can be seen as a Quasi-RNN \citep{Bradbury17} with element-wise recurrent gate functions.
Here we provide its detailed description.
Let $d_\text{in}$ and $d_\text{out}$ denote positive integers.
At each time step $t$, e-LSTM transforms an input vector $\vx(t) \in \mathbb{R}^{d_\text{in}}$ to a recurrent hidden state $\vc(t) \in \mathbb{R}^{d_\text{out}}$ as follows:
\begin{align}
\vf(t)  &= \sigma(\mF\vx(t) + \vw^{f} \odot \vc(t-1)) \\
\vz(t)  &= \tanh(\mZ\vx(t) + \vw^{z} \odot \vc(t-1)) \label{eq:lstm_in} \\
\vc(t)  &= \vf(t) \odot \vc(t-1) + (1 - \vf(t)) \odot \vz(t)
\label{eq:cell_out} \\
\vo(t) &= \sigma(\mO\vx(t)  + \mW^{o} \vc(t)) \\
\vh(t) &= \vo(t) \odot \vc(t)
\end{align}
where $\vf(t) \in \mathbb{R}^{d_\text{out}}$, $\vz(t) \in \mathbb{R}^{d_\text{out}}$, and $\vo(t) \in \mathbb{R}^{d_\text{out}}$ are activations, $\mF \in \mathbb{R}^{d_\text{out} \times d_\text{in}}$, $\mZ \in \mathbb{R}^{d_\text{out} \times d_\text{in}}$, $\mO \in \mathbb{R}^{d_\text{out} \times d_\text{in}}$ and $\mW^{o} \in \mathbb{R}^{d_\text{out} \times d_\text{out}}$ are trainable weight matrices, and finally, $\vw^{f} \in \mathbb{R}^{d_\text{out}}$ and $\vw^{z} \in \mathbb{R}^{d_\text{out}}$ are trainable weight vectors.

\end{document}